%% file: ms.tex
\documentclass{article}
\usepackage{ijcai19}

\usepackage{times}
\usepackage{soul}
\usepackage{url}
\usepackage[hidelinks]{hyperref}
\usepackage[utf8]{inputenc}
\usepackage[small]{caption}
\usepackage{graphicx}
\usepackage{amsmath}
\usepackage{booktabs}
\usepackage{algorithm}
\usepackage{algorithmic}
\usepackage{pgfplots}
\pgfplotsset{compat = newest} 
\usetikzlibrary{patterns}
\usepackage{calc}

\usepackage{graphics}
\graphicspath{{./images/}}

\usepackage{amssymb}
\usepackage{amsmath}
\usepackage{amsthm}
\usepackage{amsfonts}
\usepackage{mathtools}
\usepackage{commath}

\makeatletter
\newtheorem*{rep@theorem}{\rep@title}
\newcommand{\newreptheorem}[2]{%
\newenvironment{rep#1}[1]{%
 \def\rep@title{#2 \ref{##1}}%
 \begin{rep@theorem}}%
 {\end{rep@theorem}}}
\makeatother

\newtheorem{theorem}{Theorem}[section]
\newtheorem{lemma}[theorem]{Lemma}
\newtheorem{prop}[theorem]{Proposition}
\newtheorem{cor}[theorem]{Corollary}

\theoremstyle{definition}
\newtheorem{defn}[theorem]{Definition}

\newtheorem{ass}{Assumption}
\newtheorem{example}{Example}

\newcommand {\State}{\mathcal{S}}
\newcommand {\A}{\mathcal{A}}
\newcommand {\R}{\mathbb{R}}
\newcommand {\E}{\mathbb{E}}
\newcommand {\N}{\mathbb{N}}
\newcommand {\T}{\mathcal{T}}

\DeclarePairedDelimiterX{\inp}[2]{\langle}{\rangle}{#1, #2}

\title{Conditions on Features for Temporal Difference-Like  Methods to Converge}

\author{
Marcus Hutter $^1$
\and
Samuel Yang-Zhao$^1$
\and
Sultan J. Majeed$^1$
\affiliations
$^1$College of Engineering and Computer Science, The Australian National University
\emails
\{marcus.hutter, u6642247, sultan.majeed\}@anu.edu.au
}

\begin{document}

\maketitle

\begin{abstract}
    The convergence of many reinforcement learning (RL) algorithms with linear function approximation has been investigated extensively but most proofs assume that these methods converge to a unique solution. In this paper, we provide a complete characterization of non-uniqueness issues for a large class of reinforcement learning algorithms, simultaneously unifying many counter-examples to convergence in a theoretical framework. We achieve this by proving a new condition on features that can determine whether the convergence assumptions are valid or non-uniqueness holds. We consider a general class of RL methods, which we call \textit{natural algorithms}, whose solutions are characterized as the fixed point of a projected Bellman equation (when it exists); notably, bootstrapped temporal difference-based methods such as $TD(\lambda)$ and $GTD(\lambda)$ are natural algorithms. Our main result proves that natural algorithms converge to the correct solution if and only if all the value functions in the approximation space satisfy a certain shape. This implies that natural algorithms are, in general, inherently prone to converge to the wrong solution for most feature choices even if the value function can be represented exactly. Given our results, we show that state aggregation based features are a safe choice for natural algorithms and we also provide a condition for finding convergent algorithms under other feature constructions. 
\end{abstract}

\section{Introduction}
A longstanding goal in reinforcement learning (RL) has been to find algorithms with linear function approximation that reliably converge to the fixed point of the Bellman equations. As such the convergence of different RL methods that display such characteristics have been researched extensively. The $TD(\lambda)$ algorithm converges in an on-policy learning setting to the fixed point of a projected $\lambda$-weighted Bellman equation \cite{TsVR97}. The Residual Gradient algorithm is shown to minimise the Bellman error but suffers from double sampling \cite{Bai95}. More recently,  some temporal difference-based methods have been shown to converge with off-policy learning. The GTD2 and TDC algorithms are shown to converge to the $TD(0)$ solution under an off-policy learning setting \cite{SMP09}. These algorithms have also been extended to their bootstrapped version $GTD(\lambda)$ and shown to converge to the $TD(\lambda)$ solution \cite{Mae11}. However, a core tenet in almost all of these convergence results is the assumption that these RL methods converge to a unique solution; for example in the proof of GTD2's convergence, the matrix quantities $A$ and $C$ are assumed to be non-singular, allowing for uniqueness of solution \cite[Theorem 1]{SMP09}. 

In addition to the convergence results, pertinent counter-examples have been documented in the literature that highlight how the choice of features is crucial to convergence of RL methods \cite{Gor95,No10,Bai95,Ber94,BoM95}. Bertsekas showed that $TD(\lambda)$ with function approximation may converge to a parameter vector which generates a poor estimate of the value function (in terms of Euclidean distance) \cite{Ber94}. Tsitsiklis and Van Roy provided a counter-example showing that RL methods may diverge even when the value function is representable by the chosen features \cite{No10}. More recently, Sutton and Barto present a counter-example where methods that minimise the Bellman error may fail to learn the correct parameter value \cite[Example 11.4]{SuB18}.

In this paper, we provide a complete characterization of non-uniqueness and the potential to converge to the wrong solution for a large class of RL algorithms we call \textit{natural algorithms}. A natural algorithm is any method that can be characterized as solving for the unique fixed point (when it exists) of a projected Bellman equation. We consider all oblique projections and a Bellman equation based on the $TD(\lambda)$ Bellman operator presented in \cite{TsVR97}. Under this definition, the natural algorithms include a large spectrum of algorithms: on one end of the spectrum, the natural algorithms include bootstrapped methods such as $TD(\lambda)$ and $GTD(\lambda)$ since they are characterized by an orthogonal projection, and on the other end the natural algorithms include Bellman-error based methods which are characterized by the identity projection. We consider an RL setting with a continuous state space $\State$ and finite action space $\A$; note that a finite state space is a special case of our setup. Furthermore, we consider the infinite horizon problem for our results. 

Our main contribution is to prove that natural algorithms, even under the setting where the value function can be represented exactly by the features, are inherently prone to non-uniqueness and will converge to the wrong solution for most feature choices. Our main result is as follows:

\newtheorem*{thm_Main}{Theorem \ref{thm_Main}}
\begin{thm_Main}
    Natural algorithms converge if and only if all non-zero linear combination of the features achieve their extreme values on a sub-region of the state space that has non-zero measure under the stationary distribution.
\end{thm_Main}

\noindent
Importantly, given our characterisation, we provide some guidelines for choosing features and algorithms to avoid non-uniqueness. We show that state aggregation based features are a safe choice. We also provide a sufficient condition for algorithms to converge under other feature constructions.

This paper is organized as follows. In Section 2, we introduce some background and notation. In Section 3, we present the theory behind projected equation methods and the characteristic equation to projected Bellman equations. In Section 4, we present a detailed look at the counter-example presented by Sutton and Barto that demonstrates the non-uniqueness issues which plague Bellman-error methods \cite{SuB18}. In Section 5, we present our main results and discuss their implications, including positive feature construction examples. In Section 6, we present our framework for analyzing convergence. Finally, in Section 7, we present the idea behind the proof of Theorem \ref{thm_Main}. For brevity, most proofs and supporting results have been omitted. However, the supporting results and omitted proofs can be found in Appendices \ref{supp_results} and \ref{om_proofs} respectively.

\section{Background and Notation}
We now reiterate some background RL concepts, mathematical concepts and notation used throughout this paper. 

\subsection{RL in Continuous State Space}
We consider an agent-environment setup \cite{SuB18} where an agent follows a stationary policy $\pi$ and interacts with a Markov Decision Process (MDP). 
We assume a continuous state space $\State$ that is compact and measurable and a finite action space $\A$. For simplicity, we will assume that $\State = \R$ in all our examples. The \textit{expected reward function} is a function $R: \State \to \R$ and represents the expected reward to be received for a given state following $\pi$. At a state $s \in \State$, we assume that there is a \textit{transition density function} $T: \State \times \State \to [0, 1]$ whilst following $\pi$. Combined with an initial state $s_0$, the state sequence can be viewed as a time-homogeneous Markov process with transition kernel defined by $T(B|x) = \int_{B} T(y | x) dy ~, \forall B \in \mathcal{B}(\State)~, x \in \State$ where $\mathcal{B}(\State)$ is the Borel sigma-algebra. We consider the infinite horizon problem and thus the value function at state $s \in \State$ is defined as the total discounted expected return: $V(s) = \E \left[ \sum_{t=0}^{\infty} \gamma^{t} R(s_t) \bigg| s_0 = s \right]$, where $\gamma \in [0, 1)$ is the discount factor. By standard MDP theory, the value function satisfies the Bellman equation given by
\begin{align*}
    V(s) = R(s) + \gamma \int_{\State} T(s'|s) V(s) ds'~,
\end{align*}
for any state $s \in \State$. The Bellman operator $\T : \R^{\State} \to \R^{\State}$ is an affine linear operator on $\R^{S}$ and is defined accordingly as 
\begin{align*}
    \left(  \T V \right)(s) = R(s) + \gamma \int_{\State} T(s' | s) V(s') ds' ~.
\end{align*}
If we define $P_T$ to be an operator such that $(P_T f)(s) \coloneqq \int_{\State} T(s'|s) f(s')ds'$, we can express the Bellman operator compactly as $\T V \coloneqq R + \gamma P_T V$ for any $V \in \R^{\State}$. The Bellman equation can then be expressed as the fixed point equation $V = \T V$. 

\noindent
For an agent following a policy $\pi$ and interacting with an MDP, the state sequence can be viewed as a Markov process with transition density function $T$. Throughout this paper we assume that the state Markov process admits a stationary measure $\mu$. Under these assumptions, the value function space inherits extra geometric structure via an inner product defined with respect to $\mu$. For any $f, g \in \R^{\State}$,
\begin{align*}
    \inp{f}{g}_{\mu} \coloneqq \int_{\State} f(s)g(s) \mu(s) ds ~. 
\end{align*}
Showing that $\inp{\cdot}{\cdot}_{\mu}$ is an inner product is routine. We define the norm on the associated inner product space by $\norm{\cdot}_{\mu} = \sqrt{\inp{\cdot}{\cdot}_{\mu}}$. The set of functions in the value function space with finite $\norm{\cdot}_{\mu}$-norm is given by $L^{2}(\State, \mu) \coloneqq \{ V \in \R^{\State} : \norm{V}_{\mu} < \infty \}$. Under our assumptions, it can be shown that the value function associated with the Markov process lives in $L^{2}(\State, \mu)$. Furthermore, our approximations $\hat{V}$ of $V$ also evolve in this space. For any two functions $V, \tilde{V} \in L^2(\State, \mu)$ we say that $V$ is $\mu$-orthogonal to $\tilde{V}$, denoted by $V \perp_{\mu} \tilde{V}$, if and only if $\inp{V}{\tilde{V}}_{\mu} = 0$. 

\subsection{Linear Function Approximation}
When using linear function approximation, the value function $V$ is approximated by a linear combination of the features chosen from a finite-dimensional subspace of $L^{2}(\State, \mu)$. Formally, the approximate value function $\hat{V}$ can be written as a linear combination $\hat{V}(s, w) = \sum_{i=0}^{k} \phi_i(s)w_i , ~ \forall s \in \State$ where $w = (w_1, w_2, \ldots, w_k)^{\top} \in \R^{k}$ is a parameter vector and $\Phi = \left\{ \phi_1, \ldots, \phi_k \right\}$ is the set of features from $\State$ to $\R$ such that span$(\Phi)$ is a finite-dimensional subspace of $L^{2}(\State, \mu)$. To simplify notation, let us define $\phi(s) = ( \phi_1(s), \phi_2(s), \ldots, \phi_k(s))^{\top}$. The approximation $\hat{V}$ can then be represented compactly as an Euclidean inner product $\hat{V}(s, w) = \inp{\phi(s)}{w} = \phi^{\top}(s) w$. 

\noindent
We now define a property of linear combination of features that will be crucial in characterizing non-uniqueness.

\begin{defn}[Flat Extrema]
	Let $\varphi$ be any linear combination of the features $\Phi$, i.e. $\varphi \in \text{span}(\Phi)$, such that $\varphi \not\equiv 0$, $\varphi_{\max} \coloneqq \max_{s \in \State} \varphi(s)$ and $\varphi_{\min} \coloneqq \min_{s \in \State} \varphi(s)$. Let $\mathcal{N}_{\alpha} \coloneqq \left\{ s: \varphi(s) \geq  \alpha \varphi_{\max} \right\}$ for $\alpha \in [0, 1]$. Then we say that $\varphi$ has a \textit{flat maximum} if $\mu[\mathcal{N}_{1}] \coloneqq \int_{\mathcal{N}_{1}} \mu(s) ds > 0$.  Conversely, we say that $\varphi$ has a \textit{non-flat maximum} if $\mu[\mathcal{N}_{1}] = 0$. Similarly, we define $\mathcal{N}_{\alpha}^{-} \coloneqq \left\{ s : \varphi(s) \leq \alpha \varphi_{\min} \right\}$ and say that $\varphi$ has a \textit{flat minimum} if $\mu[ \mathcal{N}_{1}^{-} ] > 0$ and a \textit{non-flat minimum} if $\mu[ \mathcal{N}_{1}^{-} ] = 0$. 
\end{defn}

\noindent
Any linear combination of features that have flat extrema achieve their maximum (and minimum) values in regions of the state space with non-zero measure under the stationary distribution $\mu$. Conversely, linear combinations that have non-flat extrema achieve their maximum (and minimum) values in regions of the state space with non-zero measure, corresponding to sub-regions of the state space that are never visited. Figure \ref{fig:flat_max_min} gives an idea of what functions with a flat or non-flat maximum may look like. 

\begin{figure}[H]
	\begin{minipage}[b]{0.65\linewidth}
		\input{test1.tikz}
	\end{minipage}
	\begin{minipage}[b]{0.34\linewidth}
		\caption{\label{fig:flat_max_min}A sketch of two functions showing the flat/non-flat extrema property. The function $f$ has a non-flat maximum since it achieves its maximum at a point, whereas $g$ clearly has a flat maximum.}
	\end{minipage}
\end{figure}
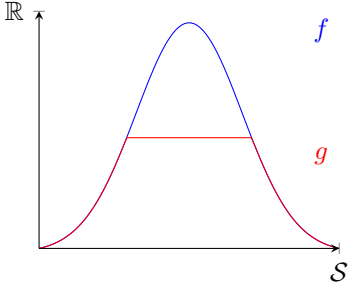

\subsection{The $TD(\lambda)$ Operator}
Tsitsiklis and Van Roy define the $TD(\lambda)$ operator $\T^{(\lambda)}$ in \cite{TsVR97} which we will adapt for our setting and derive a more compact form of the Bellman equation. For $\lambda \in [0, 1)$, the $TD(\lambda)$ operator $\T^{(\lambda)}: L^{2}(\State, \mu) \to L^{2}(\State, \mu)$ is given by
\begin{align*}
    (\T^{(\lambda)}V)(s) \coloneqq &(1-\lambda)\sum_{m=0}^{\infty}\lambda^{m}\\
    &\cdot \E \left[ \sum_{t=0}^{m} \gamma^{t} R(s_t) + \gamma^{m+1}V(s_{t+1}) \bigg| s_0 = s\right]
\end{align*}
where $V \in L^{2}(\State, \mu)$ and $s \in \State$. As Tsitsiklis and Van Roy show, the $TD(\lambda)$ operator can also be expressed as
\begin{align*}
    \T^{(\lambda)}V = (1 - \lambda) \sum_{m=0}^{\infty}\lambda^{m} \left( \sum_{t=0}^{m} (\gamma P_T)^{t}R + (\gamma P_T)^{m+1}V \right)
\end{align*}
\cite[Lemma 3]{TsVR97}. We express this operator in a more compact form as
\begin{align}
    \T^{(\lambda)}V = R^{(\lambda)} + \gamma P^{(\lambda)}_T V ~, \label{gBeq}
\end{align}
where 
\begin{align*}
    &R^{(\lambda)} \coloneqq (1 - \lambda) \sum\limits_{m=0}^{\infty}\lambda^{m} \sum\limits_{t=0}^{m} (\gamma     P_T)^{t}R~,\\
    &P^{(\lambda)}_T \coloneqq (1 - \lambda) \sum\limits_{m=0}^{\infty} (\lambda \gamma)^{m}(P_T)^{m+1}~.
\end{align*}

\noindent
The $P_T^{(\lambda)}$ operator can be seen as being a geometric average over the powers of $P_T$. We define a $\lambda$-weighted discount factor $G$ that corresponds to the discounting performed by the $P_T^{(\lambda)}$ operator:
\begin{align*}
	G \coloneqq \frac{(1-\lambda)\gamma}{1-\lambda\gamma}~.
\end{align*}
Clearly $G$ is bounded in $[0, 1)$. It is also important to note that for $\lambda = 0$ we recover the original discount factor of $\gamma$. 

As we will see later, our class of natural algorithms consists of methods that look to converge to the fixed point of a projected Bellman equation 
\begin{align}
    \hat{V} = \Pi \T^{(\lambda)} \hat{V} \label{pgBeq}
\end{align}
where $\Pi$ is a projection operator. 

\section{Projected Equation Methods}
We now introduce the theory of projected equation methods and present the characteristic equation of projected Bellman equations. We note that Bertsekas similarly covers projected Bellman equation methods \cite{Ber11}. We find it useful to reiterate the concepts here as it pertains to a continuous state space and our setup. 

\subsection{Oblique Projection Operators}
We consider the set of possible projection operators that can be applied to the Bellman equations to find an approximate solution. The projection operators that can project in \textit{any} direction are collectively known as oblique projections.  An oblique projection operator $\Pi: L^{2}(\State, \mu) \to L^{2}(\State, \mu)$ can be characterised as projecting onto $\text{im}(\Pi)$, the image of $\Pi$, and orthogonally to $\text{im}(\Pi^{*})$ where $\Pi^{*}$ is the adjoint operator of $\Pi$. The purpose of looking at projection operators is to find learnable, finite-dimensional representations of the value function. Thus, we will focus on oblique projection operators with finite-dimensional image. For bounded projection operators with finite-dimensional image, the image of the adjoint has the same dimension. 

\begin{prop}[Finite-Dimensional Projections]
	\label{prop_Riesz}
	Let $\Pi: L^{2}(\State, \mu) \to L^{2}(\State, \mu)$ be a bounded linear operator with finite-dimensional image and let $\Pi^{*}$ be its adjoint. Then the image of $\Pi^{*}$ has the same dimension as the image of $\Pi$. 
\end{prop}

\noindent
Proposition \ref{prop_Riesz} allows us to characterise the oblique projection operators in terms of two finite-dimensional subspaces. We will generalise slightly beyond bounded projection operators by considering the case where the image of the adjoint is still finite-dimensional but may not be of the same dimension as the original projection operator. As we discuss in the next sub-section, this will allow us to express the solution to projected Bellman equations as a system of linear equations. We define the set of projection operators that we are interested in as follows. 

\begin{defn}[Finite Rank Projection Operators]
	\label{natProjOp}
	Let $\Phi = \{ \phi_1, ..., \phi_k\}$ and $\Psi = \{ \psi_1, ..., \psi_n \}$. Let $\Pi: L^{2}(\State, \mu) \to L^{2}(\State, \mu)$ be an oblique projection operator such that $\text{im}(\Pi) = \text{span}(\Phi)$ and $\text{im}(\Pi^{*}) = \text{span}(\Psi)$. Then $\Pi$ can be characterised by the two sets $(\Phi, \Psi)$.
\end{defn}

\noindent
When conducting our analysis, we always assume the following. 
\begin{ass}
	\label{ass_Gen}
	Let $\Psi = \left\{ \psi_1, \ldots, \psi_n \right\}$ be a basis for $\text{im}(\Pi^{*})$. Also, assume that $\norm{\psi}_{1, \mu} \coloneqq \int_{\State} \psi_i(s)\mu(s) ds = 1$ for all $i$.
\end{ass}

\noindent
Note that Assumption \ref{ass_Gen} results in no loss of generality. Such a basis always exists for finite-dimensional spaces. Furthermore requiring $\norm{\psi_i}_{1, \mu} = 1$ is not restrictive. Since $\psi_i \in \text{im}(\Pi^{*})$ and $\text{im}(\Pi^{*}) \subset L^{2}(\State, \mu)$, we must have that $\norm{\psi_i}_{\mu} < \infty$, implying that $\norm{\psi_i}_{1, \mu}$ is also bounded and can be normalized. In the next sub-section we present the natural algorithms and how the solution to projected Bellman equations is characterised. 

\subsection{Natural Algorithms and the Solution to Projected Bellman Equations}

We now look to determine the approximate value function $\hat{V} = \phi^{\top}w \in \text{span}(\Phi)$ found as the fixed point of a projected Bellman equation. For this task, it is natural to consider a finite rank projection operator $\Pi$ characterised by $(\Phi, \Psi)$ to find $\hat{V}$ as the fixed point of $\hat{V} = \Pi \T^{(\lambda)}\hat{V}$. Since the basis functions are known, the only task left is to find an expression for the parameter vector $w$. The following proposition states that $w$ is the solution to a system of linear equations.

\begin{prop}[Characteristic Equation for Projected Bellman Equations]
	\label{ABb}
	Let $\Pi$ be a finite rank projection operator characterised by $( \Phi, \Psi )$. Suppose a unique solution exists and let $\hat{V} \in \text{span}(\Phi)$, given by $\hat{V} = \phi^{\top} w$, be the unique fixed point of the projected Bellman equation
	\begin{align*}
	\hat{V} = \Pi \T^{(\lambda)} \hat{V}.
	\end{align*}
	Let $A \in \R^{n \times k}$, $B \in \R^{n \times k}$, and $b \in \R^{n}$ be defined by
	\begin{align*}
	&A_{ij} \coloneqq \inp{\psi_i}{\phi_j}_{\mu}~,B_{ij} \coloneqq \inp{\psi_i}{P^{(\lambda)}_T \phi_j}_{\mu} ~, b_i \coloneqq \inp{\psi_i}{R}_{\mu}
	\end{align*}
	respectively for $i = 1, 2, \ldots, n$ and $j = 1, 2, \ldots, k$. Then the parameter vector $w = (w_1, \ldots, w_k)^{\top}$ is given as the solution to the system of linear equations
	\begin{align*}
	\left( A - \gamma B \right)w = b ~. 
	\end{align*}
\end{prop}

\hfil \newline
Proposition \ref{ABb} suggests that any algorithm that converges to the fixed point of a projected Bellman equation in the limit has its solution characterised by the three matrix-vector quantities $A, B$, and $b$. We denote this class of algorithms as \textit{natural algorithms}.

\begin{defn}[Natural Algorithms]
	Let $\Pi$ be a finite rank projection operator characterised by $(\Phi, \Psi)$. If an algorithm converges to the fixed point $\hat{V} = \phi^{\top}w$ of a projected Bellman equation
	\begin{align*}
	\hat{V} = \Pi \T^{(\lambda)} \hat{V} ~,
	\end{align*}
	then the algorithm is a \textit{natural algorithm}.
\end{defn}
\noindent
Some examples of natural algorithms are the $TD(\lambda)$, $GTD(\lambda)$ and Residual Gradient algorithms. This can be seen since both the $TD(\lambda)$ and $GTD(\lambda)$ algorithms converge to the $TD(\lambda)$ solution and the Residual Gradient algorithm was shown explicitly to solve an obliquely projected Bellman equation \cite{TsVR97,Mae11,Sch10}.

\section{A Counter-Example to Uniqueness}

\begin{figure}[H]
	\centering
    \input{mdp1.tikz}
    \input{mdp2.tikz}
	\caption{A counter-example by Sutton and Barto. In the following the MDPs are referred to as MDP 1 (the two state MDP) and MDP 2 (the three state MDP) respectively.}
	\label{fig:Counterexample}
\end{figure}
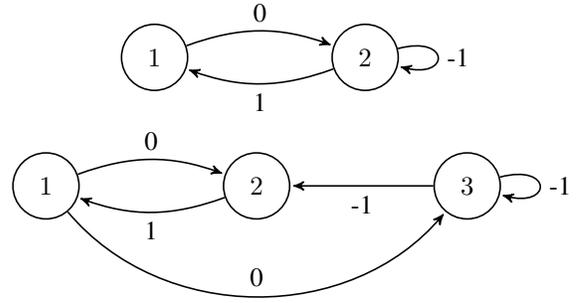

To motivate our results, we first present a counter-example highlighting some of the non-uniqueness issues that natural algorithms can face. Sutton and Barto provide a counter-example that highlights where Bellman-error based methods do not converge and experience non-uniqueness \cite[Example 11.4]{SuB18}. The counter-example considers the two Markov decision processes depicted in Figure \ref{fig:Counterexample}. The edges indicate a state transition and the labels indicate the reward received. When two edges leave a single state, we assume the transition occur with equal probability. The transition matrices of MDP 1 and MDP 2 are then given by
\begin{align*}
    T_1 = \begin{bmatrix} 0 & 1 \\ \frac{1}{2} & \frac{1}{2} \end{bmatrix}~,~ T_2 = \begin{bmatrix} 0 & \frac{1}{2} & \frac{1}{2} \\ 1 & 0 & 0 \\ 0 & \frac{1}{2} & \frac{1}{2} \end{bmatrix}~.
\end{align*}
The stationary distributions are given by $\mu_1 = (\frac{1}{3}, \frac{2}{3})^{\top}$ and $\mu_2 = (\frac{1}{3}, \frac{1}{3}, \frac{1}{3})^{\top}$ for MDP 1 and MDP 2 respectively. A simple linear function approximation mechanism is used with a two component parameter vector $w = (w_1, w_2)^{\top}$. In MDP 1, the value function can be represented exactly by $\hat{V}(1) = w_1,~ \hat{V}(2) = w_2$. In MDP 2, we assume that states 2 and 3 share a parameter value, giving $\hat{V}(1) = w_1,~ \hat{V}(2) = \hat{V}(3) = w_2$. To any RL algorithm using this feature construction, the two MDPs appear indistinguishable as the feature-reward sequence generated under the stationary distribution occur with the same probabilities. Furthermore the Bellman error, given by $E_{BE} \coloneqq \norm{(I - \gamma T) \Phi \hat{w} - R}^{2}_{\mu}$, is not a unique function of the data sample. For a parameter value $\hat{w} = 0$, the Bellman error is 0 in MDP 1 whilst it is $\frac{2}{3}$ in MDP 2. This suggests that even though an algorithm minimizing the Bellman error may converge, it may converge to the \textit{wrong} parameter vector. 

\hfill \newline
\noindent
In light of this example, we explicitly define a stronger notion of convergence to the \textit{correct} solution. The next assumption asserts that there is a true environment and that the value function can be represented.

\begin{ass}
    \label{ass_True}
    Let $R^{*}, T^{*}$ be the true environment and assume that there exists a parameter vector $w^{*}$ such that the value function $V^{*}$ can be represented as $V^{*}(s) = \phi^{\top}(s) w^{*}$.
\end{ass}

\noindent
We now define convergence as follows.
\begin{defn}(Convergence)
    An algorithm is said to converge if it converges to $w^{*}$ or, equivalently, $V^{*}$.
\end{defn}

\section{Main Results}
We now present our main results and discuss their implications. Our main theorem directly characterises convergence in terms of a property on the features.

\begin{theorem}[Flatness Condition on Features]
    \label{thm_Main}
    Natural algorithms converge if and only if all non-zero linear combinations of the features $\Phi$ have flat extrema. 
\end{theorem}

\noindent
It is important to note that Theorem \ref{thm_Main} holds for all finite rank projections onto $\text{span}(\Phi)$. The factor determining convergence is the choice of features. Theorem \ref{thm_Main} provides a restrictive condition on the possible feature choices available for natural algorithms with linear function approximation to converge. Not only are the usual assumptions of linear independence in the features necessary; it is required that \textit{all} linear combinations of the features have flat extrema. 

\hfill \newline
An immediate consequence of Theorem \ref{thm_Main} is that state aggregation methods are always safe feature construction choices. If states are aggregated such that each subset has non-zero measure under the stationary measure, all value functions in the span of these features have flat extrema. An example of this is a partitioning-based state aggregation scheme, shown in Figure \ref{fig:SA_partition}, that visibly has flat extrema. If there exist aggregated states with measure zero, these states would be unobserved under the stationary measure and would not impact the representation. This result summarised in the following corollary. 

\begin{cor}[State Aggregation]
    State aggregation is a safe feature construction choice for natural algorithms. 
\end{cor}

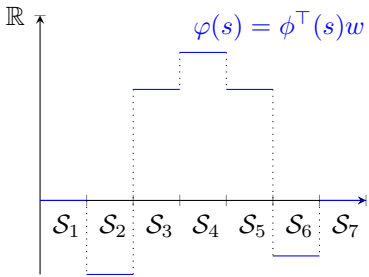
\begin{figure}[H]
    \begin{minipage}[b]{0.69\linewidth}
        \input{SA_partition.tikz}
    \end{minipage}
    \begin{minipage}[b]{0.3\linewidth}
        \caption{\label{fig:SA_partition}A state aggregation scheme which partitions the state space into non-zero measure subsets. The function $\varphi$ clearly has flat extrema.}
    \end{minipage}
\end{figure}

Though state aggregation is a sufficient choice for natural algorithms to avoid convergence issues, algorithms with other feature constructions have been shown to converge \cite{No10}. We present a condition in the next sub-section to help determine and construct convergent natural algorithms.

\subsection{A Projection Perspective}
Under Assumption \ref{ass_Gen}, the inner product between $\psi_i$ and any function $\varphi \in \text{span}(\Phi)$ can be seen as the projection of $\varphi$ onto $\psi_i$. Our next result presents a condition on these projections which can aid in determining convergent algorithms with feature constructions other than state aggregation.

\begin{theorem}[Convergent Natural Algorithms]
    \label{thm_ProjMain}
    All natural algorithms characterized by $(\Phi, \Psi)$ converge if and only if there exists an $i$ such that for all $\varphi \in \text{span}(\Phi)$
    \begin{align*}
        \inp{\psi_i}{\varphi}_{\mu} \geq G \varphi_{\max} ~\text{ or }~ \inp{\psi_i}{\varphi}_{\mu} \leq G \varphi_{\min}. 
    \end{align*}
\end{theorem}

\noindent
Theorem \ref{thm_ProjMain} guarantees a natural algorithm's convergence if it can project the extremal regions of any approximate value function. A simple case is when $\inp{\psi_i}{\varphi}_{\mu} = \varphi_{\max}$ or $\inp{\psi_i}{\varphi}_{\mu} = \varphi_{\min}$. This occurs precisely when $\varphi$ has flat extrema and $\psi_i$ projects $\varphi$ on the sub-regions of the state space where $\varphi$ achieves its extremes. An example of this is shown in Figure \ref{fig:proj_max}.

\begin{figure}[H]
    \begin{minipage}[b]{0.69\linewidth}
        \input{project.tikz}
    \end{minipage}
    \begin{minipage}[b]{0.3\linewidth}
        \caption{\label{fig:proj_max}An example where $\psi_i$ is only non-zero on the sub-region of the state space where $\varphi$ achieves its maximum value. Thus the projection achieves a value of  $\inp{\psi_i}{\varphi}_{\mu} = \varphi_{\max}$.}
    \end{minipage}
\end{figure}
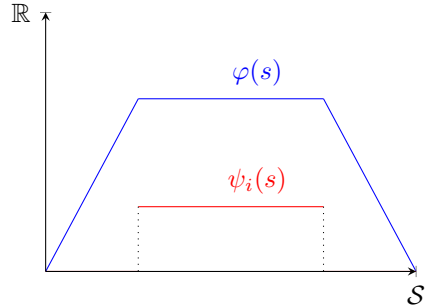

\begin{example}
    We now provide an explicit example of how Theorem \ref{thm_ProjMain} can help construct convergent natural algorithms. Consider the piece-wise linear features displayed in Figure \ref{fig:trapezoid}. Any linear combination of these features also results in piece-wise linear functions that have flat maxima. Then any natural algorithm which has the functions $\psi_i$ positive on the regions of the state space where the features achieve their flat maxima will satisfy Theorem \ref{thm_ProjMain}. Effectively, such an algorithm disregards any information about the regions of the state space that are not in the flat maxima of the features. Such an algorithm can be determined without knowledge of the value function since it only depends on the features, which are chosen apriori. 

    \begin{figure}[H]
        \begin{minipage}[b]{0.69\linewidth}
            \input{trapezoid.tikz}
        \end{minipage}
        \begin{minipage}[b]{0.3\linewidth}
            \caption{\label{fig:trapezoid} An example displaying piece-wise linear features ($\phi_1, \phi_2$) and the projection components ($\psi_1, \psi_2$). Natural algorithms of this form are guaranteed to converge.}
        \end{minipage}
    \end{figure}
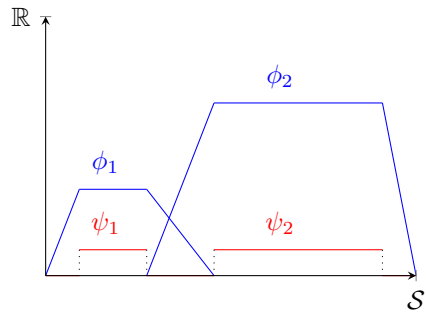
\end{example}

\begin{example}
    An example of a convergent natural algorithm that projects on the states that achieve the maximum value is the modified value iteration approach presented by Tsitsiklis and Van Roy \cite{No10}. At the outset, $K$ representative states $s_1, \ldots, s_K \in \State$ are chosen and their feature vectors $\phi(s_1), \ldots, \phi(s_K)$ are constructed. The remaining states are then chosen from within the convex hull of the feature vectors of the representative states. In this manner, the feature construction ensures that the maximum value of all approximate value functions are centred on the representative states. The modified value iteration then solves for the fixed point of 
    \begin{align*}
        \hat{V} = \Phi \Phi^{\dagger} \T (\hat{V})~,
    \end{align*}
    where $\Phi^{\dagger}$ is the left inverse of $\Phi$. In this construction, the projection operator is given by $\Phi \Phi^{\dagger}$. Since the non-representative states are composed from the representative states, the algorithm proceeds by computing only on the representative states. Thus, this method effectively takes a projection on the points of the state space that achieve the maximum value. 
\end{example}

\section{Framework of Analysis}
In this section we establish the framework we use for analysing convergence and non-uniqueness for natural algorithms. We call our framework the \textit{Bellman template}. We specifically look to capture two phenomena of non-uniqueness that were displayed in Sutton and Barto's counter-example: how natural algorithms may converge to the wrong solution even when the value function is representable and how different MDP environments with different optimal parameter vectors appear indistinguishable under projection. We formally define the Bellman Template as follows.

\begin{defn}[Bellman Template]
	Let $\Pi$ be a finite rank projection characterised by $(\Phi, \Psi)$. Let $w \in \R^{k}$, $R: \State \to \R$ such that $\norm{R}_{\infty} < \infty$, and $T: \State \times \State \to [0, 1]$ such that $\int_{\State} T(s'|s) ds' = 1$. The following constraints on $(w, R, T)$ are collectively defined as the \textit{Bellman template}:
	\begin{itemize}
		\item $(w, R, T)$ satisfy the Bellman equation
		\begin{align}
		    \hat{V} = R + \gamma P^{(\lambda)}_T \hat{V} ~, \label{constraint1}
		\end{align}
		where $\hat{V}(s) = \phi^{\top}(s) w$. 
		\item Let $A \in \R^{n \times k}$, $B \in \R^{n \times k}$ and $b \in \R^{n}$. $R$ and $T$ satisfy
		\begin{align*}
		    \inp{\psi_i}{\phi_j}_{\mu} = A_{ij}~, \inp{\psi_i}{P^{(\lambda)}_T \phi_j}_{\mu} = B_{ij}~, \inp{\psi_i}{R}_{\mu} =b_{ij}
		\end{align*}
		for $i=1, \ldots, n$ and $j = 1, \ldots, k$ respectively and $w$ satisfies
		\begin{align*}
		    (A-\gamma B)w = b~.
		\end{align*}
	\end{itemize}
	We say that a triple $(w, R, T)$ is a solution to the Bellman template if it satisfies these constraints. 
\end{defn}

\noindent
A Bellman template solution represents an MDP environment (through the expected reward function $R$ and the transition density function $T$) and its value function under the stationary policy (through the parameter vector $w$). The first constraint in the Bellman template restricts our attention to the case where the value function is exactly representable by the chosen features $\Phi$. The second constraint provides a condition to capture when different solutions to the Bellman template appear indistinguishable under projection. In particular, we consider when different solutions $(w, R, T)$ produce the same quantities $A, B$, and $b$ that characterize solutions. It may seem strange that non-uniqueness could present an issue since a solution to a projected Bellman equation is uniquely determined by $A, B$ and $b$. The crucial difference however is that we now let the environment variables, $R$ and $T$, vary.

\hfill \newline
\noindent
We define the condition of ambiguity, which represents non-uniqueness, as follows.

\begin{defn}[Ambiguity]
    \label{def_Ambig}
	\textit{Ambiguity} holds if the Bellman template has more than one $(w, R, T)$ solution that have different parameters $w$. 
\end{defn}

\noindent
Under ambiguity, different MDP environments with different optimal value functions appear the same to natural algorithms under projection. Therefore, natural algorithms fail to converge under ambiguity. 

\section{Theorem \ref{thm_Main} Proof Idea}
For brevity, we only present the key ideas behind the proof of Theorem \ref{thm_Main} here. The full proof however can be found in Appendix \ref{om_proofs}. We first present a supporting result that characterises ambiguity.

\begin{theorem}
	\label{ambig_5}
	Let $0 \not\equiv \varphi \in \text{span}(\Phi)$ and $\varphi_{\min} \coloneqq \min\limits_{s \in \State} \varphi(s)$ and $\varphi_{\max} \coloneqq \max\limits_{s \in \State} \varphi(s)$. Ambiguity holds if and only if there exists an $f: \State \to [\varphi_{\min} , \varphi_{\max}]$ such that 	
	\begin{align}
	    \int_{s \in \State} \chi_i(s) \varphi(s) ds = G \int_{s \in \State} \chi_i(s) f(s) ds~. \label{eqn4}
	\end{align}
	for all $i = 1, \ldots, n$, where $\chi_i(s) = \psi_i(s)\mu(s)$. 
\end{theorem}

\hfill \newline
\noindent
The idea behind the proof of Theorem \ref{thm_Main} is to construct a function $f$ that satisfies Theorem \ref{ambig_5}. We show that such a function $f$ exists if and only if there exists a non-zero linear combination of the features $\varphi$ that has non-flat extrema. Then by the equivalence given in Theorem \ref{ambig_5}, ambiguity holds meaning natural algorithms will fail to converge. Taking the contrapositive then gives Theorem \ref{thm_Main}.

To construct a suitable function, first consider the function $\tilde{f} \coloneqq \frac{1}{G} \varphi$. Clearly $\tilde{f}$ satisfies (\ref{eqn4}). For $\varphi_{\max} < 0$, $\tilde{f}$ satisfies the upper bound on the range as $\tilde{f}(s) \leq \varphi_{\max}$ for all $s \in \State$. Similarly, in the case where $\varphi_{\min} > 0$, $\tilde{f}$ satisfies the lower bound on the range. However when $\varphi_{\max} > 0$, $\tilde{f}$ exceeds the upper bound. Again in similar fashion, when $\varphi_{\min} < 0$, $\tilde{f}$ exceeds the lower bound. We now look to construct a function from $\tilde{f}$ that does not exceed the bounds in these cases whilst still satisfying (\ref{eqn4}). We will focus on the $\varphi_{\max} > 0$ case, noting that $\varphi_{\min} < 0$ is treated the same way. We consider capping $\tilde{f}$ at $G \varphi_{\max}$ and spreading the `cut' pinnacle across the basis functions $\chi_i$ in a way that satisfies (\ref{eqn4}). Consider
\begin{align*}
    \bar{f}(s) \coloneqq \frac{1}{G} \varphi(s) - \delta(s)~,
\end{align*}
where $\delta(s) \coloneqq \max \{ 0, \frac{1}{G} \varphi(s) - G \varphi_{\max} \}$ is the cut pinnacle. Now define $f \coloneqq \bar{f} + g$, where $g$ is some function. We look to find $g$ as the pinnacle $\delta$ projected onto the basis functions $\chi_i$ in such a way that (\ref{eqn4}) is satisfied. It can be shown that (\ref{eqn4}) is equivalent to
\begin{align*}
    \int_{\State} \chi_i(s) \delta(s) ds = \int_{\State} \chi_i(s) g(s) ds
\end{align*}
for all $i = 1, \ldots, n$. Thus there are $n$ constraints for $g$ to satisfy. To get a gauge on whether spreading $\delta$ in this fashion may be possible, consider the following. The height of the portion of $\frac{1}{G}\varphi$ that exceeds $\varphi_{\max}$ is of order O$(1-G)$. Also, if $\varphi$ has a non-flat maximum and $G$ is sufficiently close to 1, this portion has `mass' o$(1- G)$ since $\mu\left[ \mathcal{N}_{G} \right]$ goes to 0 as $G$ goes to 1. As $g$ spreads the mass of $\delta$ over $\chi_i$ independently from $G$, it is of order o$(1- G)$. Thus for $G$ sufficiently close to 1, $g(s) < \delta(s)$. Thus it seems plausible that $f = \bar{f} + g = \frac{1}{G} \varphi - \delta + g$ would not surpass the upper bound of $\varphi_{\max}$. We employ an analogous construction to find a function $f^{-}$ that satisfies the lower bound for the case where  $\varphi_{\min} < 0$. Finally, we are able to combine the different functions we construct to satisfy the range conditions for each case in Theorem \ref{ambig_5} and ultimately show our result. Figure \ref{fig:impression} provides a sketch of the construction we employ in this proof. 
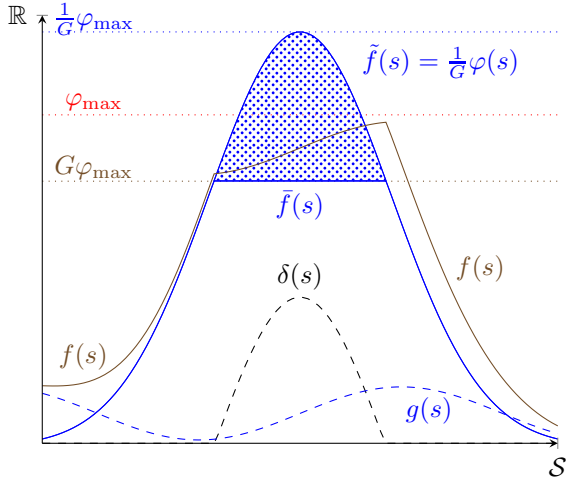
\begin{figure}
    \centering
    \input{test.tikz}
    \caption{An impression of our construction of $f$ in Theorem \ref{thm_Main} that satisfies the upper bound. The function $g$ can be viewed as spreading out $\delta$ over the $\chi_i$ functions such that $f$ does not exceed the upper bound.}
    \label{fig:impression}
\end{figure}

\section{Conclusion}
We have established that natural algorithms are inherently prone to fail without careful consideration of the choice of features. In particular, natural algorithms converge if and only if the features chosen only allow for linear combinations with flat extrema. We also presented a condition from a projection perspective that can help determine the convergence of natural algorithms as well. Given our results, we justify that state aggregation features are sufficient for \textit{all} natural algorithms to converge. We also provide a condition under which natural algorithms with other feature constructions can converge if they project upon the extreme regions of the features. In doing so, we show how a convergent natural algorithm can be constructed from our result as well as arguing for the convergence of the modified value iteration approach presented in \cite{No10}. 

It is important to note that our results begin where the assumptions in previous convergence proofs do not hold. As an example, $TD(\lambda)$ is known to converge on-policy but counter-examples exist for the off-policy case \cite{TsVR97}. In our analysis, the off-policy case is subsumed by the projection operators we consider as we do not restrict them to be non-expansions with respect to the $\norm{\cdot}_{\mu}$-norm. Thus, our result also implies the divergence of Q-learning. 

We note that the natural algorithm class, while extensive, does not cover all known RL algorithms with linear function approximation. In particular the ETD algorithm, introduced in \cite{SMW16}, does not fall within our natural algorithm class.  The ETD algorithm includes an extra interest function $i$ that alters the visiting probabilities of states, meaning we are no longer working with the stationary distribution $\mu$. 

An important factor in determining whether our results will hold in practice is the choice of discount factor. Throughout our analysis, the discount factor plays an important role in defining the extrema regions. As the discount factor moves away from one and towards zero, it becomes less likely that the non-flat extrema property will occur. Thus the discount factor determines the degree to which feature choices that deviate from flat extrema allow natural algorithms to converge. Also, our analysis centres on a strict notion of convergence to the true value function. Investigating whether our analysis can extend to characterise non-uniqueness when considering approximate value functions is an interesting open question. 

\appendix
\clearpage
\section{Supporting Results}
\label{supp_results}
In this section we present the supporting results we use to prove Theorem \ref{thm_Main} and Theorem \ref{thm_ProjMain} as well as any omitted proofs here. 

The first lemma places ambiguity in relation to a non-trivial null-space of the linear system of equations that characterize solutions to projected Bellman equations. A non-trivial null-space is exactly what governs non-uniqueness in a system of linear equations; the only added difficulty is the difference that varying environment variables $R$ and $T$ presents. 

\begin{lemma}
	\label{ambig_1}
	Ambiguity holds if and only if there exists $0 \not\equiv v \in \R^{k}$ such that $\left( A - \gamma B \right)v = 0$.
\end{lemma}
\begin{proof}
    Under ambiguity, there exists another solution to the Bellman template $(w^{0}, R^{0}, T^{0})$ such that $w^{0} \neq w^{*}$. From the derivation of Proposition \ref{ABb}, we can show that both $w = w^{0}$ and $w = w^{*}$ satisfy
    \begin{align*}
    (A - \gamma B) w = b ~.
    \end{align*}
    Then trivially, $(A - \gamma B)v = 0$ holds since we can take $v = w^{*} - w^{0}$. Now consider the reverse. Assume that $(A - \gamma B)v = 0$ holds for some $v \neq 0$. For $\xi > 0$, let us define
    \begin{align*}
    &w^{\xi} = w^{*} + \xi v,~ T^{\xi} = T^{*},~ R^{\xi} = \phi^{\top} w^{\xi} - \gamma P^{(\lambda)}_{T^{\xi}} \phi^{\top} w^{\xi} ~.
    \end{align*}
    Then $(w^{\xi}, R^{\xi}, T^{\xi})$ satisfies the generalised Bellman equation since
    \begin{align*}
    \T^{(\lambda)} \phi^{\top}(s) w &= R^{\xi}(s) + \gamma P^{(\lambda)}_{T^{\xi}} \phi^{\top}(s) w^{\xi}\\
    &= \phi^{\top}(s) w^{\xi} ~,
    \end{align*}
    which is precisely the left-hand side of the generalised Bellman equation. Let $A^{\xi} \in \R^{n \times k}$, $B^{\xi} \in \R^{n \times k}$ and $b \in \R^{n}$ be given by
    \begin{align*}
    A^{\xi}_{ij} = \inp{\psi_i}{\phi_j}_{\mu},~ B^{\xi}_{ij} = \inp{\psi_i}{P^{(\lambda)}_{T^{\xi}} \phi_j}_{\mu},~ b^{\xi}_i = \inp{\psi_i}{R}_{\mu}~,
    \end{align*}
    for  $i=1, \ldots, n$ and $j = 1, \ldots, k$. We trivially have that $A^{\xi} = A$ and we note that $B^{\xi} = B$ since $T^{\xi} = T^{*}$. By Proposition \ref{ABb}, applying the projection $\Pi$ characterised by $(\Phi, \Psi)$ onto the generalised Bellman equation induces the following
    \begin{align*}
    \left( A^{\xi} - \gamma B^{\xi} \right) w^{\xi} = b^{\xi} ~.
    \end{align*}
    We then have the following derivation
    \begin{align*}
    b^{\xi} &= \left( A^{\xi} - \gamma B^{\xi} \right) w^{\xi}\\
    &= \left( A^{\xi} - \gamma B^{\xi} \right) \left( w^{*} + \xi v \right)\\
    &\overset{(a)}{=} \left( A^{\xi} - \gamma B^{\xi} \right) w^{*}\\
    &= b ~. 
    \end{align*}
    where (a) follows since $(A - \gamma B)v = 0$. Thus we have that $b^{\xi} = b$ as well. Thus, for any $\xi > 0$, $(w^{\xi}, R^{\xi}, T^{\xi})$ satisfies the generalised Bellman equation as well and so ambiguity holds.
\end{proof}

\hfil \newline
The next corollary is effectively a restatement of Lemma \ref{ambig_1} that will be easier to work with later. 

\begin{cor}
	\label{ambig_2}
	For all $i = 1, \ldots, n$, let $\chi_i(s) \coloneqq \psi_i(s) \mu(s)$ where $\mu$ is the stationary distribution. Then ambiguity holds if and only if there exists $T$ and $0 \not\equiv \varphi \in \text{span}(\Phi)$
	such that for all $i = 1, \ldots, n$,
	\begin{align*}
	\int_{\State} \chi_i(s) \big( I - \gamma P^{(\lambda)}_T \big) \varphi(s) ds = 0 ~.
	\end{align*} 
\end{cor}
\begin{proof}
	From Theorem \ref{ambig_1}, we have that ambiguity holds if and only if there exists $0 \not\equiv v \in \R^{k}$ such that
	\begin{align*}
    	\left( A - \gamma B \right)v = 0 ~.
    	\end{align*}
    	Now for all $i=1, \ldots, n$ and $j = 1, \ldots, k$
    	\begin{align*}
    	A_{ij} - \gamma B_{ij} &= \int_{\State} \psi_i(s) \mu(s) \phi_j(s) - \gamma \psi_i(s) \mu(s) P^{(\lambda)}_T \phi_j(s) ds\\
    	&= \int_{\State} \chi_i(s) \left( I - \gamma P^{(\lambda)}_T \right) \phi_j(s) ds
	\end{align*}
	for some $T$. Now note that $A_{i \cdot}$ and $B_{i \cdot}$ are row vectors of $A$ and $B$ for all $i = 1, \ldots, n$. Then since $\left( A - \gamma B \right)v = 0$, we have for all $i = 1, \ldots, n$ the following derivation
	\begin{align*}
	0 &= \left( A_{i \cdot} - \gamma B_{i \cdot} \right) v\\
	&= \sum_{j = 0}^{k} \left( A_{ij} - \gamma B_{ij} \right) v_j\\
	&\overset{(a)}{=} \sum_{j = 0}^{k} \int_{\State} \chi_i(s) \left( I - \gamma P^{(\lambda)}_T \right) \phi_j(s) v_j ds\\
	&\overset{(b)}{=} \int_{\State} \chi_i(s) \left( I - \gamma P^{(\lambda)}_T \right) \sum_{j = 0}^{k} \phi_j(s) v_j ds\\
	&= \int_{\State} \chi_i(s) \left( I - \gamma P^{(\lambda)}_T \right) \varphi(s) ds ~.
	\end{align*}
	Here in (a) we substituted in the derivation of $A_{ij} - \gamma B_{ij}$ from above and in (b) we used the Fubini-Tonelli theorem to swap the sum and the integral. Since we've only looked at equivalences, our if and only if result holds. 	
\end{proof}

\hfil \newline
The next lemma presents a useful characterisation of non-flat extrema as the $\lambda$-weighted discount factor approaches the limit. 

\begin{lemma}
	\label{gamma_1}
	Let $0 \not\equiv \varphi \in \text{span}(\Phi)$. Then $\varphi$ has non-flat maximum if and only if $\mu[\mathcal{N}_{G}] \to 0$ as $G \to 1$. Similarly, $\varphi$ has non-flat minimum if and only if $\mu[\mathcal{N}_{G}^{-}] \to 0$ as $G \to 1$. 
\end{lemma}

\begin{proof}
	We will only prove that $\varphi$ has non-flat maximum if and only if $\mu[\mathcal{N}_{G}] \to 0$ as $G \to 1$ noting that adapting the proof for the non-flat minimum case is routine. \newline
	
	\noindent
	We first show the forward direction. Since $\varphi$ has a non-flat maximum, we have that $\mu[\mathcal{N}_{1}] = 0$. Also $\mu[\State]$ can be written as
	\begin{align*}
	    \mu[\State] = \mu[\State \backslash \mathcal{N}_{G} + \mathcal{N}_{G}] = \mu[\State \backslash \mathcal{N}_{G} ] + \mu[\mathcal{N}_{G}] ~. 
	\end{align*}
	\noindent
	Re-arranging and taking the limit as $G$ goes to 1 gives
	\begin{align*}
	\lim_{G \to 1} \mu[\mathcal{N}_G] = \lim_{G \to 1} (\mu[\State] - \mu[S\backslash \mathcal{N}_G]) = 0~.
	\end{align*}
	Thus the forward direction holds. Now consider the reverse direction. Suppose $\mu[\mathcal{N}_G] \to 0$ as $G \to 1$. Then trivially $\lim_{G \to 1} \mu[\mathcal{N}_G] = \mu[\mathcal{N}_{1}] = 0$. Thus the result holds. 
\end{proof}

\noindent
The next result shows that $P^{(\lambda)}_T$ is a non-expansion.
\begin{prop}
	\label{nexp}
	The operator $P^{(\lambda)}_T$ is a non-expansion with respect to $\norm{\cdot}_{\mu}$. 
\end{prop}

\begin{proof}
	We first show that $P_T$ is a non-expansion before showing the result. To see that $P_T$ is a non-expansion note that
	\begin{align*}
    	\norm{P_T V}^{2}_{\mu} &= \inp{P_T V}{P_T V}_{\mu}\\
    	&= \int_{\State} \mu(s) \left( \int_{\State'} T(s'|s)V(s')ds' \right)^{2} ds\\
    	&\overset{(a)}{\leq} \int_{\State} \mu(s) \int_{\State'}T(s'|s)V(s')^{2}ds'ds\\
    	&\overset{(b)}{=} \int_{\State'} \int_{\State} \mu(s) T(s' | s) V(s')^{2}ds ds'\\
    	&\overset{(c)}{=} \int_{\State'} \mu(s') V(s')^{2}ds'\\
    	&= \norm{V}^{2}_{\mu}~.
	\end{align*}
	Here (a) follows by Jensen's inequality, (b) follows by the Tonelli-Fubini theorem, and (d) follows since $\mu$ is the stationary distribution. Since the quadratic function is monotonically increasing on $\R_{+}$, we have that $\norm{P_T V}_{\mu} \leq \norm{V}_{\mu}$ and so $P_T$ is non-expansive. 
	
	\noindent
	Now to see that $P^{(\lambda)}_T$ is non-expansive, we have the following derivation:
	\begin{align*}
	\norm{P^{(\lambda)}_T V}_{\mu} &= \norm{(1 - \lambda) \sum_{m=0}^{\infty} (\gamma \lambda)^{m} P_T^{m+1} V}_{\mu}\\
	&\overset{(a)}{\leq} \norm{(1-\lambda) \sum_{m=0}^{\infty} (\gamma \lambda)^{m}V}_{\mu}\\
	&= \frac{1 - \lambda}{1 - \gamma \lambda} \norm{V}_{\mu}\\
	&\overset{(b)}{\leq} \norm{V}_{\mu}~.
	\end{align*}
	Here (a) follows as $P_T$ is a non-expansion and (b) follows since $\frac{1 - \lambda}{1 - \gamma \lambda} \leq 1$ as $\lambda \geq \gamma \lambda$. Thus $P^{(\lambda)}_T$ is a non-expansion. 
\end{proof}

\section{Omitted Proofs}
\label{om_proofs}
\subsection*{Proposition \ref{prop_Riesz} Proof}
\begin{proof}
    Let $l_y(x) = \inp{\Pi x}{y}_{\mu}$. Then by the Riesz representation theorem, there exists $z_y$ such that $l_y(x) = \inp{x}{z_y}_{\mu}$ and by definition, $\Pi^{*}y = z_y$. Now note that if $y \in \text{im}(\Pi)^{\perp}$, then $y \in \text{ker}(\Pi^{*})$ and so $z_y = 0$. Now let $z \in L^{2}(\State, \mu)$. Since $L^{2}(\State, \mu)$ can be decomposed into the direct sum of $\text{im}(\Pi)$ and $\text{im}(\Pi)^{\perp}$, then there exists $z_1 \in \text{im}(\Pi)$ and $z_2 \in \text{im}(\Pi)^{\perp}$ such that $z = z_1 + z_2$. Applying $\Pi^{*}$ to $z$ then gives
    \begin{align*}
    \Pi^{*} z = \Pi^{*}z_1 ~.
    \end{align*}
    Thus since $z_1 \in \text{im}(\Pi)$, which is finite-dimensional, the image of $\Pi^{*}$ must also have the same dimensions. 
\end{proof}

\subsection*{Proposition \ref{ABb} Proof}

\begin{proof}
    Since $\Pi$ projects orthogonally to $\text{im}(\Pi^{*})$ and onto $\text{im}(\Pi)$, we have that $\T \hat{V} - \Pi \T \hat{V} \perp_{\mu} \text{im}(\Pi^{*})$. So for all $i = 1, \ldots, n, ~ \inp{\psi_i}{\T \hat{V} - \Pi \T \hat{V}}_{\mu} = 0$. Substituting $\hat{V}$ for $\Pi \T \hat{V}$, expanding and re-arranging gives
    \begin{align*}
    0 &= \inp{\psi_i}{\T\hat{V} - \hat{V}}_{\mu}\\
    &= \inp{\psi_i}{  \left( R + \gamma P_T^{(\lambda)} \phi^{\top}w \right) - \phi^{\top} w }_{\mu}\\
    &= \inp{\psi_i}{R}_{\mu} - \sum_{j=0}^{k} \inp{\psi_i}{\left( \phi_j - \gamma P_T^{(\lambda)}\phi_j \right) w_j}_{\mu}\\
    &= \inp{\psi_i}{R}_{\mu} - \sum_{j=0}^{k} \inp{\psi_i}{\left( I - \gamma P_T^{(\lambda)} \right)\phi_j w_j}_{\mu}~.
    \end{align*}
    Re-arranging now gives
    \begin{align*}
    \sum_{j=0}^{k} \inp{\psi_i}{\left( I - \gamma P_T^{(\lambda)} \right)\phi_j w_j}_{\mu} = \inp{\psi_i}{R}_{\mu} ~. 
    \end{align*}
    Then given our definitions of $A, B$, and $b$, we have a linear system of $k$ equations given by $\left(A - \gamma B\right) w = b$.
\end{proof}

\subsection*{Theorem \ref{ambig_5} Proof}
\begin{proof}
	By Corollary \ref{ambig_2}, ambiguity holds if and only if there exists a$T$ and $0 \not\equiv \varphi \in \text{span}(\Phi)$ such that for all $i$
	\begin{align}
	\int_{s \in S} \chi_i(s) \varphi(s) ds = \gamma \int_{s \in S} \chi_i(s) P^{(\lambda)}_T \varphi(s) ds~. \label{eqn_a}
	\end{align}
	We define $\tilde{P}_T^{(\lambda)}$ as
	\begin{align*}
		\tilde{P}_T^{(\lambda)} \coloneqq  \frac{1 - \lambda\gamma}{1-\lambda} P_T^{(\lambda)}~. 
	\end{align*}
	The $\tilde{P}_T^{(\lambda)}$ operator can be viewed as a re-normalised version of $P_T^{(\lambda)}$. It can then be seen that $\tilde{P}_T^{(\lambda)} \varphi(s)$ is bound between $\varphi_{\min}$ and $\varphi_{\max}$ since
	
	\begin{align*}
		&\tilde{P}_T^{(\lambda)} \varphi_{\min} = \frac{1 - \lambda\gamma}{1-\lambda} P_T^{(\lambda)} \varphi_{\min} = \varphi_{\min}~,
	\end{align*}
	\noindent
	and
	\begin{align*}
		&\tilde{P}_T^{(\lambda)} \varphi_{\max} = \frac{1 - \lambda\gamma}{1-\lambda} P_T^{(\lambda)} \varphi_{\max} = \varphi_{\max}~.
	\end{align*}
	\noindent 
	Now (\ref{eqn_a}) can be re-written as
	\begin{align*}
		\int_{s \in S} \chi_i(s) \varphi(s) ds = G \int_{s \in S} \chi_i(s) \tilde{P}^{(\lambda)}_T \varphi(s)~.
	\end{align*}
	So in the forward direction, we can simply take $f(s) = \tilde{P}^{(\lambda)}_T \varphi(s)$. Now in the reverse direction, we note that any function $f: \State \to [\varphi_{\min}, \varphi_{\max}]$ can be represented as $\tilde{P}^{(\lambda)}_T \varphi(s)$ for some $T$. Thus the result holds. 
\end{proof}

\subsection*{Theorem \ref{thm_Main} Proof}
We look to construct a function $f$ that satisfies Theorem \ref{ambig_5} as well as
\begin{align}
    \int_{s \in \State} \chi_i(s) \varphi(s) ds = G \int_{s \in \State} \chi_i(s) f(s) ds \label{eqn_rep}
\end{align}
for $i = 1, \ldots, n$. We show that such a function $f$ exists in each case if and only if there exists $0 \not\equiv \varphi \in \text{span}(\Phi)$ that has non-flat extrema. Then by the equivalence given in Theorem \ref{ambig_5}, ambiguity holds. Taking the contrapositive then gives Theorem \ref{thm_Main}. We first look to find a function $f$ that satisfies (\ref{eqn_rep}) and the upper bound on the range.

\hfill \newline
Consider $\tilde{f}(s) = \frac{1}{G} \varphi(s)$. For $\varphi_{\max} < 0$, $\tilde{f}$ clearly satisfies (\ref{eqn_rep}) and the upper bound on the range; but for $\varphi_{\max} > 0$, $\tilde{f}$ exceeds the upper bound. Instead, we look to construct a function from $\tilde{f}$ whereby we `cut' off the portion that exceeds the upper bound and project it across the functions $\chi_i$ for $i = 1, \ldots, n$ to satisfy the upper bound whilst still satisfying (\ref{eqn_rep}). Consider
\begin{align*}
	\bar{f}(s) \coloneqq \frac{1}{G} \varphi(s) - \delta(s)~,
\end{align*}
where $\delta(s) \coloneqq \max\{ 0, \frac{1}{G} \varphi(s) - G \varphi_{\max} \}$ is the cut pinnacle. By construction, $\bar{f}$ satisfies the upper bound. Now let us define
\begin{align*}
	f(s) \coloneqq \bar{f}(s) + g(s) ~.
\end{align*}
We look to find $g$ such that $f$ satisfies (\ref{eqn_rep}) and the upper bound. The following derivation derives a linear system of equations that constrain the function $g$ such that $f$ satisfies (\ref{eqn_rep}). Starting from (\ref{eqn_rep}), we have
\begin{align*}
	\int_{S} \chi_i(s) \varphi(s) ds &= G \int_{S} \chi_i(s) f(s) ds\\
	&\overset{(a)}{=} \gamma \int_{S} \chi_i(s) \left( \frac{1}{G} \varphi(s) - \delta(s) + g(s) \right) ds ~.
\end{align*}
Here (a) follows by the definition of $f$ and $\bar{f}$. Cancelling out $\int_{S} \chi_i(s) \varphi(s) ds$ and $G$ from both sides of the equation gives
\begin{align}
	\int_{S} \chi_i(s) \delta(s)ds = \int_{S} \chi_i(s) g(s)ds ~. \label{eqn_21}
\end{align}
We now note that the functions $\chi_1, \ldots, \chi_n$ forms a linearly independent set. To see this, suppose that for all $s \in \State$
\begin{align*}
	b_1 \chi_1(s) + b_2 \chi_2(s) + \ldots + b_n \chi_n(s) = 0 
\end{align*}
and $b_1 , \ldots b_n$ are not all equal to 0. Then since $\chi_i(s) = \mu(s) \psi_i(s)$, and $\mu(s) > 0$, we must have
\begin{align*}
	b_1 \psi_1(s) + b_2 \psi_2(s) \ldots + b_n \psi_n(s) = 0 ~.
\end{align*}
This is a contradiction since the set of functions $\Psi = \{ \psi_1, \ldots, \psi_n \}$ is a linearly independent set. Thus, let $g$ be given by
\begin{align*}
	g(s) = \sum_{j = 1}^{n} \chi_j(s) a_j ~, \forall s \in \State 
\end{align*}
where $a_i \in \R$ for all $i$. Then from (\ref{eqn_21}) we have
\begin{align*}
	\int_{S} \chi_i(s) \delta(s)ds &= \int_{\State} \chi_i(s) \sum_{j = 1}^{n} \chi_j(s) a_j ds\\
	&\overset{(a)}{=} \sum_{j = 1}^{n} a_j \int_{\State} \chi_i(s) \chi_j(s) ds		
\end{align*}
where in (a) we swapped the summation and the integrand by the Fubini-Tonelli theorem. Now as a notational shorthand, let $\inp{f}{g} \coloneqq \int_{\State} f(s) g(s) ds$. Then we have
\begin{align*}
	\sum_{j = 1}^{n} a_j \int_{\State} \chi_i(s) \chi_j(s) ds = \sum_{j = 1}^{n} a_j \inp{\chi_i}{\chi_j}~.
\end{align*}
Let $X \in \R^{n \times n}$ and $\bar{\delta} \in \R^{n}$ be defined by
\begin{align*}
	X_{ij} &= \inp{\chi_i}{\chi_j} ~, \quad ~ i, j = 1, \ldots, n \\
	\bar{\delta}_i &= \inp{\chi_i}{\delta} ~, \quad ~ i = 1, \ldots, n
\end{align*}
respectively. Together $a$, X and $\bar{\delta}$ form a system of linear equations given by
\begin{align*}
	X a = \bar{\delta}
\end{align*}
Note that since $\chi_1, \ldots, \chi_n$ is a set of linearly independent functions, $X$ is full rank and thus invertible. We can express $a$ as
\begin{align*}
	a = X^{-1} \bar{\delta} ~.
\end{align*}
Let $\chi(s) = \left( \chi_1(s), \ldots, \chi_n(s) \right)$. We can now express $g$ as
\begin{align*}
	g(s) = \chi(s) X^{-1} \bar{\delta} ~, \quad \forall s \in \State ~.
\end{align*}
We now explicitly define three infinity norms we will use to bound $g$. For a function $f(s) = (f_1(s), \ldots, f_n(s))$, vector $u \in \R^{n}$, and matrix $A \in \R^{n \times n}$, the norms are given by
\begin{align*}
	\norm{f}_{\infty} &\coloneqq \max_{1 \leq i \leq n} \sup_{s \in \State} \abs{f_i(s)} ~,\\
	\norm{u}_{\infty} &\coloneqq \max_{1 \leq i \leq n} \abs{u_i} ~,\\
	\norm{A}_{\infty} &\coloneqq \sup_{y \neq 0} \frac{\norm{Ay}_{\infty}}{\norm{y}_{\infty}} = \max_{j} \sum_{i=1}^{n} \abs{A_{ij}} ~.
\end{align*}
Under these norm definitions, we see that $\norm{X^{-1}}_{\infty} < \infty$ since $X$ is invertible. The function $\chi$ has its infinity norm given by $\norm{\chi}_{\infty} = \max_{1 \leq i \leq n} \sup_{s \in \State} \abs{\chi_i(s)}$. To see that this is finitely bounded, recall that for any $i$, $\int_{\State} \chi_i(s) ds = 1$. We can then split the $\chi_i$ into two functions $\chi_i^{+}$ and $\chi_i^{-}$ where $\chi_i^{+}$ is the same value as $\chi_i$ when it is positive and $\chi_i^{-}$ is the same value as $\chi_i$ when it is negative. Then since
\begin{align*}
	\int_{\State} \chi_i(s) ds = \int_{\State} \chi_i^{+}(s) - \int_{\State} \chi_i^{-}(s) ds = 1 ~,
\end{align*}
it must be the case that both individual integrals are finite. Thus $\abs{\chi_i(s)} < \infty$ for all $i$ and $s$. We now look to derive a bound for $\bar{\delta}$. Consider the set $\mathcal{N}_{G^{2}}$ given by
\begin{align*}
	\mathcal{N}_{G^{2}} \coloneqq \left\{ s \in \State : \varphi(s) \geq G^{2} \varphi_{\max} \right\} ~.
\end{align*}
Note that $\delta(s) > 0$ if and only if $s \in \mathcal{N}_{G^{2}}$. Let $\mathbf{1}_{\mathcal{N}_{G^{2}}}$ be the characteristic function for $\mathcal{N}_{G^{2}}$. Then for all $i = 1, \ldots, n$, we have the following derivation
\begin{align*}
	\int_{\State} \chi_i(s) &\delta(s) ds \overset{(a)}{=} \int_{\State} \psi_i(s) \delta(s) \mu(s) ds\\
	&\overset{(b)}{\leq} \left( \frac{1}{G} \varphi_{\max} - G \varphi_{\max} \right) \int_{\State} \mathbf{1}_{\mathcal{N}_{G^{2}}}(s) \psi_i(s) \mu(s) ds\\
	&\overset{(c)}{=} \left( \frac{1}{G} \varphi_{\max} - G \varphi_{\max} \right) \psi_i(s_{\max}) \int_{\mathcal{N}_{G^{2}}} \mu(s)ds\\
	&\overset{(d)}{=} \left( \frac{1}{G} \varphi_{\max} - G \varphi_{\max} \right) \psi_i(s_{\max}) \mu\left[ \mathcal{N}_{G^{2}} \right] ~.
\end{align*}
Here (a) follows by definition of $\chi_i$ and (b) follows by definition of $\delta$. In (c), we let $s_{\max}$ denote the value of $s \in \State$ that $\psi_i$ achieves a maximum. Finally, (d) follows by definition of $\mu\left[ \mathcal{N}_{G^{2}} \right]$. Thus, $\bar{\delta}$ can be bounded by
\begin{align*}
	\norm{\bar{\delta}}_{\infty} = \max_{1 \leq i \leq n} \abs{\bar{\delta}} \leq \varphi_{\max} \left( \frac{1}{G} - G \right) \psi_i(s_{\max}) \mu\left[ \mathcal{N}_{G^{2}} \right] ~.
\end{align*}
Combining the bounded quantities, we have that $g$ is bounded by
\begin{align*}
	\norm{g}_{\infty} &\leq \norm{\chi}_{\infty} \norm{X^{-1}}_{\infty} \varphi_{\max} \left( \frac{1}{G} - G \right) \psi_i(s_{\max}) \mu\left[ \mathcal{N}_{G^{2}} \right]\\
	&\leq C \cdot \varphi_{\max} \left( \frac{1}{G} - G \right) \mu\left[ \mathcal{N}_{G^{2}} \right] 
\end{align*}
where to simplify notation we let $C$ be the constant defined as $C = \norm{\chi}_{\infty} \norm{X^{-1}}_{\infty} \psi_i(s_{\max}) $. Note that by definition, $\bar{f}$ is bounded by $G \varphi_{\max}$. As a result, we can now bound $f$ from above by
\begin{align*}
	f(s) &= \bar{f}(s) + g(s)\\
	&\leq G \varphi_{\max} + \norm{g}_{\infty}\\
	&\leq \varphi_{\max} \left( G + C \left(\frac{1}{G} - G \right) \mu\left[ \mathcal{N}_{G^{2}} \right] \right) ~.
\end{align*}
Now $\varphi_{\max} \left( G + C \left(\frac{1}{G} - G \right) \mu\left[ \mathcal{N}_{G^{2}} \right] \right) \leq \varphi_{\max}$ if and only if 
\begin{align}
G + C \left(\frac{1}{G} - G \right) \mu\left[ \mathcal{N}_{G^{2}} \right] \leq 1 ~. \label{eqn_22}
\end{align}
This preceding inequality holds if and only if 
\begin{align*}
\mu\left[ \mathcal{N}_{G^{2}}\right] \leq \frac{1 - G}{C \left(\frac{1}{G} - G \right)} ~,
\end{align*}
which is equivalent to
\begin{align*}
    \mu\left[ \mathcal{N}_{G^{2}} \right] \leq \frac{G}{C\left( 1 + G \right)}~.
\end{align*}
Since $\varphi$ has a non-flat maximum, by Lemma \ref{gamma_1} we have that as $G \to 1$
\begin{align*}
\mu\left[ \mathcal{N}_{G^{2}} \right] \to 0 ~,
\end{align*}
whilst we have as $G \to 1$,
\begin{align*}
\frac{G}{C \left( 1 + G \right)} \to \frac{1}{2 C} ~.
\end{align*}
Thus, $\mu\left[ \mathcal{N}_{G^{2}} \right] \leq \frac{G}{C\left( 1 + G \right)}$ if $\varphi$ has non-flat maximum. Thus $f(s) \leq \varphi_{\max}$ for all $s \in \State$ and $f$ satisfies the upper bound.

By an analogous argument to the above, we can also find a construction to satisfy the lower bound in the different cases and (\ref{eqn_rep}). For $\varphi_{\min} > 0$, $f = \frac{1}{G} \varphi(s)$ satisfies (\ref{eqn_rep}) and the lower bound. For $\varphi_{\min} < 0$, we can find a function $f^{-}$ that satisfies (\ref{eqn_rep}) and has range greater than the lower bound given $\varphi$ has non-flat minimum. We define $f^{-}$ by 
\begin{align*}
	f^{-}(s) = \tilde{f}(s) - g^{-}(s)~, \forall s \in \State
\end{align*}
where $g^{-}$ is some function chosen such that $f^{-}$ satisfies (\ref{eqn_rep}). The function $\tilde{f}(s)$ is given by
\begin{align*}
	\tilde{f}(s) = \frac{1}{G}\varphi(s) + \delta^{-}(s) ~,
\end{align*}
where $\delta^{-}(s)$ is given by
\begin{align*}
	\delta^{-}(s) = \max \{0, G \varphi_{\min} - \frac{1}{G} \varphi(s) \} ~.
\end{align*}
By construction, $\tilde{f}$ satisfies the lower bound. We now look to derive a system of linear equations to constrain $g^{-}$ such that (\ref{eqn_rep}) is satisfied. Starting from (\ref{eqn_rep}) we have
\begin{align*}
	\int_{\State} \chi_i(s) \varphi(s) ds &= G \int_{\State} \chi_i(s) f^{-}(s) ds\\
	&= G \int_{\State} \chi_i(s) \left( \frac{1}{G}\varphi(s) + \delta^{-}(s) - g^{-}(s) \right) ds ~.
\end{align*}
Cancelling out from both sides $\int_{\State} \chi_i(s) \varphi(s)$ ds and re-arranging gives
\begin{align*}
	\int_{\State} \chi_i(s) \delta^{-}(s) ds = \int_{\State} \chi_i(s) g^{-}(s) ds ~.
\end{align*}
We let $g^{-}$ be given by
\begin{align*}
	g^{-}(s) = \sum_{j=1}^{n}\chi_j(s) a_j ~.
\end{align*}
Let $\bar{\delta}^{-} \in \R^{n}$ be given by
\begin{align*}
	\bar{\delta}^{-} = \inp{\chi_i}{\delta^{-}}~, \quad i = 1, \ldots, n ~.
\end{align*}
Then $g^{-}$ is given by
\begin{align*}
	g^{-}(s) = \chi(s) X^{-1} \delta^{-} ~.
\end{align*}
We now look to bound $\bar{\delta}^{-}$. Let $\mathcal{N}^{-}_{G^{2}} \coloneqq \{ s \in \State: \varphi(s) \leq G^{2} \varphi_{\min}\}$. For all $i = 1, \ldots, n$ we have the following derivation
\begin{align*}
	\int_{\State} \chi_i(s) \delta^{-}(s) ds &= \int_{\State} \psi_i(s) \delta^{-}(s) \mu(s) ds\\
	&\overset{(a)}{\leq}  \varphi_{\min}\left( G - \frac{1}{G} \right) \int_{\mathcal{N}^{-}_{G^{2}}} \psi_i(s) \mu(s) ds\\
	&\overset{(b)}{\leq} \varphi_{\min}\left( G - \frac{1}{G} \right) \psi_i(s_{\min}) \int_{\mathcal{N}^{-}_{G^{2}}} \mu(s) ds\\
	&\overset{(c)}{=} \varphi_{\min}\left( G - \frac{1}{G} \right) \psi_i(s_{\min}) \mu\left[\mathcal{N}^{-}_{G^{2}}\right] ~.
\end{align*}
In (a), we used the fact that $\delta^{-}(s)$ is only greater than 0 for $s \in \mathcal{N}^{-}_{G^{2}}$ and that it is upper bounded by $\varphi_{\min}\left( G - \frac{1}{G} \right)$. In (b) we define $s_{\min}$ as the value of $s \in \State$ where $\psi_i$ achieves a minimum. Finally, (c) follows by definition of $\mu\left[\mathcal{N}^{-}_{G^{2}}\right]$. Having bound $\int_{\State} \chi_i(s) \delta^{-}(s) ds$ for all $i$, we have that $\bar{\delta}^{-}$ is bound in the infinity norm
\begin{align*}
	\bar{\delta}^{-} &\leq \varphi_{\min}\left( G - \frac{1}{G} \right) \psi_i(s_{\min}) \mu\left[\mathcal{N}^{-}_{G^{2}}\right]~.
\end{align*}
Thus, $g^{-}(s)$ is bounded in the infinity norm by
\begin{align*}
	\norm{g^{-}}_{\infty} &\leq \norm{\chi}_{\infty} \norm{X^{-1}}_{\infty} \norm{\delta^{-}}_{\infty}\\
	&= D \varphi_{\min}\left( G - \frac{1}{G}\right) \mu\left[\mathcal{N}^{-}_{G^{2}}\right]
\end{align*}
where to simplify notation we define the constant $D \coloneqq \norm{\chi}_{\infty} \norm{X^{-1}}_{\infty} \psi_i(s_{\min})$. Given that $\tilde{f}(s)$ is bound below by $G \varphi_{\min}$, we have that $f^{-}(s)$ is bounded as follows.
\begin{align*}
	f^{-}(s) &= \tilde{f}^{-}(s) - g^{-}(s)\\
	&\geq G \varphi_{\min} - D \varphi_{\min}\left( G - \frac{1}{G}\right) \mu\left[\mathcal{N}^{-}_{G^{2}}\right]~.
\end{align*}
Now $G \varphi_{\min} - D \varphi_{\min}\left( G - \frac{1}{G}\right) \mu\left[\mathcal{N}^{-}_{G^{2}}\right] \geq \varphi_{\min}$ if and only if
\begin{align*}
	G + D \left(\frac{1}{G} + G \right) \mu\left[\mathcal{N}^{-}_{G^{2}}\right] \leq 1 ~.
\end{align*}
since $\varphi_{\min} < 0$. Re-arranging shows that this inequality holds if and only if
\begin{align*}
	\mu\left[\mathcal{N}^{-}_{G^{2}}\right] &\leq \frac{1 - G}{D \left(\frac{1}{G} - G \right)}\\
	&= \frac{G}{D \left(1 + G \right)} ~.
\end{align*}
Since $\varphi$ has non-flat minimum, we have that $\mu\left[\mathcal{N}^{-}_{G^{2}}\right] \to 0$ as $G \to 1$ whereas $\frac{G}{D \left(1 + G \right)} \to \frac{1}{2D}$. Thus, the inequality holds if $\varphi$ has non-flat minimum.

Note that $- \delta(s) + g(s)$ is only non-zero on $\mathcal{N}_{G^{2}}$ and $\delta^{-}(s) - g^{-}(s)$ is only non-zero on $\mathcal{N}^{-}_{G^{2}}$ and the two sets do not intersect. Thus, combining the functions $f$ and $f^{-}$ we have
\begin{align*}
	f(s) \coloneqq \frac{1}{\gamma} \varphi(s) - \delta(s) + g(s) + \delta^{-}(s) - g^{-}(s)
\end{align*}
that satisfies the range and (\ref{eqn4}) if $\varphi$ has non-flat extrema. Thus, ambiguity holds if and only if there exists a $0 \not\equiv \varphi \in \text{span}(\Phi)$ with non-flat extrema. Taking the contrapositive proves Theorem \ref{thm_Main}.

\subsection*{Theorem \ref{thm_ProjMain} Proof}
By Corollary \ref{ambig_2}, we have that ambiguity holds if and only if there exists a $T$ and $0 \not\equiv \varphi \in \text{span}(\Phi)$ such that for $i = 1,\ldots, n$, $\int_{\State} \chi_i(s) \varphi(s) ds = \gamma \int_{\State} \chi_i(s) P_T^{(\lambda)}\varphi(s) ds$. By expanding $P_T^{(\lambda)}$ we get
\begin{align*}
    P_T^{(\lambda)} \varphi = (1-\lambda) \sum_{m=0}^{\infty} (\gamma \lambda)^{m} P_T^{m+1} \varphi ~.
\end{align*}
The function $\varphi$ is bounded between $\varphi_{\min}$ and $\varphi_{\max}$. Furthermore, we have that $P_T \varphi_{\max} = \varphi_{\max}$ and $P_T \varphi_{\min} = \varphi_{\min}$. Thus for any $k > 0$ the following holds:
\begin{align}
    \varphi_{\min} \leq P_T^{k} \varphi \leq \varphi_{\max}~. \label{f1}
\end{align}
Furthermore, we note that
\begin{align*}
    (1- \lambda) \sum_{m=0}^{\infty} (\lambda \gamma)^{m} \varphi_{\min} = \frac{(1-\lambda)}{1-\gamma \lambda} \varphi_{\min}~,\\
    (1- \lambda) \sum_{m=0}^{\infty} (\lambda \gamma)^{m} \varphi_{\max} = \frac{(1-\lambda)}{1-\gamma \lambda} \varphi_{\max}
\end{align*}
by the geometric series. Then using the inequality in \ref{f1}, $P_T^{(\lambda)}$ is bounded by:
\begin{align*}
    \frac{(1-\lambda)}{1-\gamma \lambda} \varphi_{\min} \leq P_T^{(\lambda)} \varphi \leq \frac{(1-\lambda)}{1-\gamma \lambda} \varphi_{\max} ~. 
\end{align*}
Now note that 
\begin{align*}
    \int_{\State} \chi_i(s) \varphi_{\min}(s) ds &= \varphi_{\min} \int_{\State} \psi_i(s) \mu(s) ds\\
    &= \varphi_{\min}
\end{align*}
where the last equality holds due to Assumption \ref{ass_Gen}. Similarly, $\int_{\State} \chi_i(s) \varphi_{\max}(s) ds = \varphi_{\max}$. Thus the following holds:
\begin{align*}
    G \varphi_{\min} \leq
    \gamma \int_{\State} \chi_i(s) P_T^{(\lambda)} \varphi(s) ds \leq G \varphi_{\max} ~. 
\end{align*}
Now we note that,
\begin{align*}
    \gamma \int_{\State} \chi_i(s) P_T^{(\lambda)} \varphi(s) ds &\overset{(a)}{=} \int_{\State} \chi_i(s) \varphi(s)ds \\
    &\overset{(b)}{=} \int_{\State} \psi_i(s) \varphi(s) \mu(s) ds \\
    &= \inp{\psi_i}{\varphi}_{\mu} ~.
\end{align*}
Here (a) follows by Corollary \ref{ambig_2} and (b) follows by definition of $\chi_i$. Substituting in $\inp{\psi_i}{\varphi}_{\mu}$ give us that ambiguity holds if and only if there exists $0 \not\equiv \varphi \in \text{span}(\Phi)$ such that for all $i$, 
\begin{align*}
    G \varphi_{\min} \leq \inp{\psi_i}{\varphi}_{\mu} \leq G \varphi_{\max} ~.
\end{align*}
Taking the contrapositive of this statement gives the desired result. 

\newpage
\section{List of Notation}

\renewcommand{\thefootnote}{\fnsymbol{footnote}}

In the following, let $\mathcal{H}$ and $\mathcal{W}$ be two vector spaces and $A: \mathcal{H} \to \mathcal{W}$ a linear transformation. 
\\

\noindent\textbf{Notation}

\newcommand{\nttn}[2]{\item[{\ \makebox[3.18cm][l]{#1}}]{#2}}
\begin{list}{}{ \setlength{\leftmargin}{3.4cm}
                \setlength{\labelwidth}{3.4cm}}

\nttn{RL} {Reinforcement Learning.}

\nttn{MDP} {Markov Decision Process.}

\nttn{$TD$} {Temporal difference.}

\nttn{$\State$}{The state space.}

\nttn{$\gamma$} {The discount factor, i.e. $\gamma \in [0, 1)$.}
\nttn{$G$}{The $\lambda$-weighted discount factor given by $G \coloneqq \frac{(1-\lambda)\gamma}{1-\lambda\gamma}$. }

\nttn{$G \to 1$} {$G$ approaches a value of 1.}

\nttn{$\mu$} {The stationary distribution.}

\nttn{$\T$} {The Bellman operator.}

\nttn{$\Phi$} {The set of chosen features: $\Phi \coloneqq \{ \phi_1, \phi_2, \ldots, \phi_k \}$, $k \in \N$. }

\nttn{$\phi(s)$} {A feature vector for a given state $s \in \State$: $\phi(s) = (\phi_1(s), \ldots, \phi_k(s))^{\top}$.}

\nttn{$\varphi_{\min}$}{$\varphi_{\min} \coloneqq \min_{s \in \State} \varphi(s)$. }

\nttn{$\varphi_{\max}$}{$\varphi_{\max} \coloneqq \max_{s \in \State} \varphi(s)$. }

\nttn{$T$}{The transition density function.}

\nttn{$P_T$}{An operator such that for a function $f: \State \to \R$, $(P_T f)(s) \coloneqq \int_{\State} T(s'|s) f(s') ds'$.}

\nttn{$P_T^{(\lambda)}$}{An operator defined as $P_T^{(\lambda)} \coloneqq (1-\lambda) \gamma \sum_{m=0}^{\infty} (\lambda \gamma)^{m} (P_T)^{m+1}$. }

\nttn{$\inp{f}{g}_{\mu}$}{An inner product defined as $\inp{f}{g}_{\mu} \coloneqq \int_{\State} f(s)g(s)\mu(s) ds$.}

\nttn{$\norm{f}_{\mu}$}{The norm defined by the inner product $\inp{f}{g}_{\mu}$, i.e. $\norm{f}_{\mu} \coloneqq \sqrt{\inp{f}{f}_{\mu}}$.}

\nttn{$\norm{f}_{1,\mu}$}{$\norm{f}_{1,\mu} \coloneqq \int_{\State} f(s)\mu(s)ds$. }

\nttn{$\R$}{The set of real numbers.}

\nttn{$\text{im}(A)$}{The image of $A$, that is the set $\text{im}(A) \coloneqq \{ Ax : x \in \mathcal{H} \}$. }

\nttn{$\text{ker}(A)$}{The kernel of $A$, that is the set $\text{ker}(A) \coloneqq \{ x \in \mathcal{H} : Ax = 0 \}$.}

\nttn{$\text{span}(B)$}{The span of $B$.}

\nttn{$\text{dim}(\mathcal{H})$}{The dimension of $\mathcal{H}$.}

\nttn{$x \perp y$}{x is perpendicular to y.}

\nttn{$x \perp_{\mu} y$}{Given an inner product $\inp{\cdot}{\cdot}_{\mu}$, this notation denotes that $x$ is perpendicular to $y$ with respect to $\inp{\cdot}{\cdot}_{\mu}$.}

\nttn{$\mu\left[ \R \right]$}{$\mu\left[ \R \right] \coloneqq \int_{\R} \mu(s) ds$.}

\nttn{$\mathbf{1}_{B}(s)$}{The characteristic function for the set $B$.}

\nttn{$\mathbf{1}(s' = s)$}{The indicator function.}

\end{list}

\hfill \newpage
\input{ms.bbl}

\end{document}

%% file: test1.tikz
\usetikzlibrary{patterns}
\begin{tikzpicture}
\begin{axis}[
clip=false,
axis x line=bottom, xtick={2}, xticklabel={$\mathcal{S}$},
axis y line=left, ytick={0.21}, yticklabels={$\mathbb{R}$}, ymax=0.21,
no marks,
samples=100,
domain=-2:2,
axis on top,xmin=-2
, width = 1\textwidth, height = .85\textwidth
]
\addplot+[solid, color=blue] {0.2*exp(-x*x)}
[xshift=50pt][yshift=-3pt] node[pos=0.5]{$f$};
\addplot+[solid, color=red] {min(0.2*exp(-x*x),0.1)}
[xshift=50pt][yshift=-7pt] node[pos=0.5]{$g$};
\end{axis}
\end{tikzpicture}

%% file: mdp1.tikz
\usetikzlibrary{arrows,automata}
\begin{tikzpicture}[->,>=stealth',shorten >=1pt,auto,node distance=2.8cm,
                    semithick]
  \tikzstyle{every state}=[fill=white,text=black]


    \node[state]        (A)                     {$1$};
    \node[state]        (B) [right of =A]       {$2$};

    \path (A) edge [bend left=20] node {0} (B)
          (B) edge [bend left=20] node {1} (A)
              edge [loop right] node {-1} (B);
\end{tikzpicture}

%% file: mdp2.tikz
\usetikzlibrary{arrows,automata}
\begin{tikzpicture}[->,>=stealth',shorten >=1pt,auto,node distance=2.8cm,
                    semithick]
  \tikzstyle{every state}=[fill=white,text=black]


    \node[state]        (A)                     {$1$};
    \node[state]        (B) [right of =A]       {$2$};
    \node[state]        (C) [right of =B]       {$3$};

    \path (A) edge [bend left=20] node {0} (B)
              edge [bend right=50] node {0} (C)
          (B) edge [bend left=20] node {1} (A)
          (C) edge                node {-1} (B)
              edge [loop right]   node {-1} (C);
            
\end{tikzpicture}

%% file: SA_partition.tikz
\usetikzlibrary{patterns}
\pgfplotsset{compat=1.6}
\pgfplotsset{soldot/.style={color=blue,only marks,mark=*}} \pgfplotsset{holdot/.style={color=blue,fill=white,only marks,mark=*}}

\begin{tikzpicture}
\begin{axis}[
    clip=false,
    axis x line=middle, xtick={2,4,6,8,10,12,14}, xticklabels={$\mathcal{S}_1 ~\quad~$, $\mathcal{S}_2 ~\quad~$, $\mathcal{S}_3 ~\quad~$, $\mathcal{S}_4 ~\quad~$, $\mathcal{S}_5 ~\quad~$, $\mathcal{S}_6 ~\quad~$, $\mathcal{S}_7 ~\quad~$},
    axis y line=left, ytick={5}, yticklabels={$\mathbb{R}$}, ymax=5,
    no marks,
    width = 1\textwidth, height = .85\textwidth
]   
    \addplot[blue, domain=0:2]{0};
    \addplot[blue, domain=2:4]{-2};
    \addplot[blue, domain=4:6]{3};
    \addplot[blue, domain=6:8]{4}
    [xshift=20pt][yshift=10pt] node[pos=1]{$\varphi(s) = \phi^{\top}(s) w$};
    \addplot[blue, domain=8:10]{3};
    \addplot[blue, domain=10:12]{-1.5};
    \addplot[blue, domain=12:14]{0};
    \draw[dotted] (axis cs:2,0) -- (axis cs:2,-2);
    \draw[dotted] (axis cs:4,3) -- (axis cs:4,-2);
    \draw[dotted] (axis cs:6,4) -- (axis cs:6,3);
    \draw[dotted] (axis cs:8,4) -- (axis cs:8,3);
    \draw[dotted] (axis cs:10,3) -- (axis cs:10,-1.5);
    \draw[dotted] (axis cs:12,0) -- (axis cs:12,-1.5);

\end{axis}
\end{tikzpicture} 

%% file: project.tikz
\usetikzlibrary{patterns}
\pgfplotsset{compat=1.6}
\pgfplotsset{soldot/.style={color=blue,only marks,mark=*}} \pgfplotsset{holdot/.style={color=blue,fill=white,only marks,mark=*}}


\begin{tikzpicture}
\begin{axis}[
    clip=false,
    axis x line=bottom, xtick={8}, xticklabels={$\State$},
    axis y line=left, ytick={3}, yticklabels={$\mathbb{R}$}, ymax=3,
    no marks,
    width = 1.1\textwidth, height = .85\textwidth,
    samples=100,
    domain=0:8,
    axis on top,
    xmin=0
]   
    \addplot[blue, domain=0:2]{x};
    \addplot[blue, domain=2:6]{2}
    [xshift=-25pt][yshift=10pt] node[pos=1]{$\varphi(s)$};
    \addplot[blue, domain=6:8]{8-x};
    \addplot[red, domain=0:2]{0};
    \addplot[red, domain=2:6]{.75}
    [xshift=-25pt][yshift=10pt] node[pos=1]{$\psi_i(s)$};
    \addplot[red, domain=6:8]{0};
    \draw[dotted] (axis cs:2,.75) -- (axis cs:2,0);
    \draw[dotted] (axis cs:6,.75) -- (axis cs:6,0);

\end{axis}
\end{tikzpicture} 

%% file: trapezoid.tikz
\usetikzlibrary{patterns}
\pgfplotsset{compat=1.6}
\pgfplotsset{soldot/.style={color=blue,only marks,mark=*}} \pgfplotsset{holdot/.style={color=blue,fill=white,only marks,mark=*}}


\begin{tikzpicture}
\begin{axis}[
    clip=false,
    axis x line=bottom, xtick={22}, xticklabels={$\State$},
    axis y line=left, ytick={3}, yticklabels={$\mathbb{R}$}, ymax=3,
    no marks,
    width = 1.1\textwidth, height = .85\textwidth,
    samples=100,
    domain=0:14,
    axis on top,
    xmin=0
]   
    
    \addplot[blue, domain=0:2]{0.5*x};
    \addplot[blue, domain=2:6]{1}
    [xshift=10pt][yshift=10pt] node[pos=0]{$\phi_1$};
    \addplot[blue, domain=6:10]{-0.25*x+2.5};
    \addplot[blue, domain=6:10]{0.5*x-3};
    \addplot[blue, domain=10:20]{2}
    [xshift=25pt][yshift=10pt] node[pos=0]{$\phi_2$};
    \addplot[blue, domain=20:22]{-x+22};
    
    
    \addplot[red, domain=0:2]{0};
    \addplot[red, domain=2:6]{0.3}
    [xshift=10pt][yshift=10pt] node[pos=0]{$\psi_1$};
    \addplot[red, domain=6:10]{0};
    \addplot[red, domain=10:20]{0.3}
    [xshift=25pt][yshift=10pt] node[pos=0]{$\psi_2$};
    \addplot[red, domain=20:22]{0};

    \draw[dotted] (axis cs:2,.3) -- (axis cs:2,0);
    \draw[dotted] (axis cs:6,.3) -- (axis cs:6,0);
    \draw[dotted] (axis cs:10,.3) -- (axis cs:10,0);
    \draw[dotted] (axis cs:20,.3) -- (axis cs:20,0);


\end{axis}
\end{tikzpicture} 

%% file: test.tikz
\usetikzlibrary{patterns}
\begin{tikzpicture}
\begin{axis}[
clip=false,
axis x line=bottom, xtick={2}, xticklabel={$\mathcal{S}$},
axis y line=left, ytick={1.3}, yticklabels={$\mathbb{R}$}, ymax=1.3,
no marks,
samples=100,
domain=-2:2,
axis on top,xmin=-2]
\addplot+[dotted] {1.25}
[yshift=5pt]node[pos=0.1]{$\frac{1}{G}\varphi_{\max}$};
\addplot+[dotted] {1.0}
[yshift=5pt]node[pos=0.1]{$\varphi_{\max}$};
\addplot+[dotted] {0.8}
[yshift=5pt]node[pos=0.1]{$G\varphi_{\max}$};
\addplot+[solid, color=blue] {1.25*exp(-x*x)}
[xshift=40pt] node[pos=0.56]{$\tilde{f}(s)= \frac{1}{G} \varphi(s)$};
\addplot+[solid, color=blue] {min(1.25*exp(-x*x),0.8)}
[yshift=-9pt] node[pos=0.5]{$\bar f(s)$};
\addplot[pattern=crosshatch dots, pattern color=blue,draw=blue, samples=100]
{0.8+sqrt(1.25*exp(-x*x)-0.8)^2} --cycle;
\addplot+[solid] {min(1.25*exp(-x*x),0.8)+0.08*sin(deg(2*x))+0.1}
[yshift=10pt] node[pos=0.07]{$f(s)$}
[xshift=10pt] node[pos=0.8]{$f(s)$};
\addplot+[dashed] {max(1.25*exp(-x*x)-0.8,0.01)}
[yshift=8pt] node[pos=0.5]{$\delta(s)$};
\addplot+[dashed] {0.08*sin(deg(2*x))+0.1}
[yshift=-8pt] node[pos=0.75]{$g(s)$};
\end{axis}
\end{tikzpicture}

%% file: ms.bbl
\newcommand{\etalchar}[1]{$^{#1}$}

%% file: ms.bbl
\begin{thebibliography}{SMP{\etalchar{+}}09}

\bibitem[Bai95]{Bai95}
L~Baird.
\newblock {Residual Algorithms: Reinforcement Learning with Function
  Approximation}.
\newblock {\em Proceedings of the Twelfth International Conference on Machine
  Learning}, pages 30--37, 1995.

\bibitem[Ber95]{Ber94}
D.~P. Bertsekas.
\newblock {A Counterexample to Temporal Difference Learning}.
\newblock {\em Neural Computation}, 7:270--279, 1995.

\bibitem[Ber11]{Ber11}
D.~P. Bertsekas.
\newblock {\em {Dynamic Programming and Optimal Control 3rd Edition, Volume
  II}}.
\newblock Massachusetts Institute of Technology, 2011.

\bibitem[BM95]{BoM95}
Justin~A. Boyan and Andrew~W. Moore.
\newblock {Generalization in Reinforcement Learning: Safely Approximating the
  Value Function}.
\newblock pages 369--376, 1995.

\bibitem[Gor95]{Gor95}
G.~J. Gordon.
\newblock {Stable Function Approximation in Dynamic Programming}.
\newblock {\em Proceedings of the Twelfth International Conference on Machine
  Learning}, pages 261--268, 1995.

\bibitem[Mae11]{Mae11}
H.~R. Maei.
\newblock {\em {Gradient Temporal-Difference Learning Algorithms}}.
\newblock PhD thesis, University of Alberta, 2011.

\bibitem[SB18]{SuB18}
R.~S. Sutton and A.~G. Barto.
\newblock {\em {Reinforcement Learning: An Introduction}}.
\newblock The MIT Press, 2018.

\bibitem[Sch10]{Sch10}
B.~Scherrer.
\newblock {Should one compute the Temporal Difference fix point or minimize the
  Bellman Residual? The unified oblique projection view}.
\newblock 2010.
\newblock arXiv:1011.4362.

\bibitem[SMP{\etalchar{+}}09]{SMP09}
R.~S. Sutton, H.~R. Maei, D.~Precup, S.~Bhatnagar, D.~Silver,
  C.~Szepesv\'{a}ri, and E.~Wiewiora.
\newblock {Fast Gradient Descent Methods for Temporal-Difference Learning With
  Linear Function Approximation}.
\newblock {\em Proceedings of the 26th Annual International Conference on
  Machine Learning}, pages 993--1000, 2009.

\bibitem[SMW16]{SMW16}
R.~S. Sutton, A.~R. Mahmood, and M.~White.
\newblock {An Emphatic Approach to the Problem of Off-policy
  Temporal-Difference Learning}.
\newblock {\em Journal of Machine Learning Research}, 17:1--29, 2016.

\bibitem[TVR96]{No10}
J.~Tsitsiklis and B.~Van~Roy.
\newblock {Feature-Based Methods for Large Scale Dynamic Programming}.
\newblock {\em Machine Learning}, 22:59--94, 1996.

\bibitem[TVR97]{TsVR97}
J.~Tsitsiklis and B.~Van~Roy.
\newblock {An Analysis of Temporal-Difference Learning with Function
  Approximation}.
\newblock {\em IEEE Transactions on Automatic Control}, 42(5):674--690, 1997.

\bibitem[Yu15]{Yu15}
Huizhen Yu.
\newblock {On Convergence of Emphatic Temporal-Difference Learning}.
\newblock {\em JMLR: Workshop and Conference Proceedings}, 40:1--28, 2015.

\end{thebibliography}
